\documentclass{article}

\PassOptionsToPackage{numbers, compress}{natbib}

\usepackage[final]{nips_2017}

\usepackage[utf8]{inputenc} 
\usepackage[T1]{fontenc}    
\usepackage{hyperref}       
\usepackage{url}            
\usepackage{booktabs}       
\usepackage{amsfonts}       
\usepackage{nicefrac}       
\usepackage{microtype}      

\usepackage{graphicx} 
\usepackage{paralist}
\usepackage{amsbsy,amssymb,amsmath,amsthm}
\usepackage{bm}
\usepackage{algorithm}
\usepackage{algorithmic}
\usepackage{subcaption}
\usepackage{tikz}
\usepackage{accents}
\usepackage{paralist}
\usepackage{microtype}
\usepackage{siunitx}
\newcolumntype{H}{>{\setbox0=\hbox\bgroup}c<{\egroup}@{}}
\usepackage{color}
\usepackage{xspace}
\usepackage{enumitem}

\newcommand{\Acal}{\mathcal{A}}
\newcommand{\Gcal}{\mathcal{G}}
\newcommand{\Pcal}{\mathcal{P}}
\newcommand{\Ical}{\mathcal{I}}
\newcommand{\Mcal}{\mathcal{M}}

\newcommand{\Ocal}{\mathcal{O}}
\newcommand{\Ucal}{\mathcal{U}}

\newcommand{\EE}{\mathbb{E}}
\newcommand{\RR}{\mathbb{R}}
\newcommand{\II}{\mathbb{I}}
\newcommand{\PP}{\mathbb{P}}
\newcommand{\NN}{\mathbb{N}}

\newcommand{\fhat}{\hat{f}}
\newcommand{\Fbar}{\bar{F}}

\newcommand{\cb}{\mathbf{c}}
\newcommand{\gb}{\mathbf{g}}
\renewcommand{\sb}{\mathbf{s}}
\newcommand{\ub}{\mathbf{u}}
\newcommand{\wb}{\mathbf{w}}
\newcommand{\xb}{\mathbf{x}}
\newcommand{\yb}{\mathbf{y}}
\newcommand{\zb}{\mathbf{z}}
\newcommand{\one}{\mathbf{1}}
\newcommand{\zero}{\mathbf{0}}

\newcommand{\bbeta}{\bm{\beta}}
\newcommand{\bgamma}{\bm{\gamma}}

\newcommand{\tnorm}[1]{\lVert #1 \rVert_2}
\newcommand{\tgnorm}[1]{\lVert #1 \rVert_G}

\newcommand*{\argmin}{\mathop{\mathrm{argmin}}}

\ifx\theorem\undefined
\newtheorem{theorem}{Theorem}

\newtheorem{lemma}[theorem]{Lemma}
\newtheorem{proposition}[theorem]{Proposition}

\fi

\ifx\BlackBox\undefined
\newcommand{\BlackBox}{\rule{1.5ex}{1.5ex}}  
\fi

\ifx\proof\undefined
\newenvironment{proof}{\par\noindent{\bf Proof\ }}{\hfill\BlackBox\\[2mm]}
\fi

\newcommand{\ubar}[1]{\underaccent{\bar}{#1}}

\title{Stochastic Submodular Maximization:\\The Case of Coverage Functions}

\author{
  Mohammad Reza Karimi\\
  Department of Computer Science\\
  ETH Zurich\\
  \texttt{mkarimi@ethz.ch} \\
  \And
  Mario Lucic \\
  Department of Computer Science \\
  ETH Zurich\\
  \texttt{lucic@inf.ethz.ch} \\
  \AND
  Hamed Hassani \\
  Department of Electrical and Systems Engineering \\
  University of Pennsylvania \\
  \texttt{hassani@seas.upenn.edu} \\
  \And
  Andreas Krause \\
  Department of Computer Science \\
  ETH Zurich \\
  \texttt{krausea@ethz.ch}\\
}

\begin{document}
\maketitle

\begin{abstract}
Stochastic optimization of {\em continuous} objectives is at the heart of modern machine learning.  However, many important problems are of {\em discrete} nature and often involve {\em submodular} objectives. We seek to unleash the power of stochastic continuous optimization, namely stochastic gradient descent and its variants, to such discrete problems. We first introduce the problem of \emph{stochastic submodular optimization}, where one needs to optimize a submodular objective which is given as an expectation. Our model captures situations where the discrete objective arises as an empirical risk (e.g., in the case of exemplar-based clustering), or is given as an explicit stochastic model (e.g., in the case of influence maximization in social networks).  By exploiting that common extensions {\em act linearly} on the class of submodular functions, we employ projected stochastic gradient ascent and its variants in the continuous domain, and perform rounding to obtain discrete solutions. We focus on the rich and widely used family of weighted coverage functions. We show that our approach yields solutions that are guaranteed to match the optimal approximation guarantees, while reducing the computational cost by several orders of magnitude, as we demonstrate empirically.
\end{abstract}

\section{Introduction}
Submodular functions are discrete analogs of convex functions. They arise naturally in many areas, such as the study of graphs, matroids, covering problems, and facility location problems. These functions are extensively studied in operations research and combinatorial optimization \citep{krause2012submodular}. Recently, submodular functions have proven to be key concepts in other areas such as machine learning, algorithmic game theory, and social sciences. As such, they have been applied to a host of important problems such as modeling valuation functions in combinatorial auctions, feature and variable selection \cite{krause05near}, data summarization \cite{lin2011class}, and influence maximization \cite{kempe}.

Classical results in submodular optimization consider the {\em oracle model} whereby the access to the optimization objective is provided through a black box --- an oracle. However, in many applications, the objective has to be estimated from data and is subject to stochastic fluctuations. In other cases the value of the objective may only be obtained through simulation. As such, the exact computation might not be feasible due to statistical or computational constraints. As a concrete example, consider the problem of {\em influence maximization} in social networks \cite{kempe}. The objective function is defined as the expectation of a stochastic process, quantifying the size of the (random) subset of nodes influenced from a selected seed set. This expectation cannot be computed efficiently, and is typically approximated via random sampling, which introduces an  error in the estimate of the value of a seed set. Another practical example is the {\em exemplar-based clustering} problem, which is an instance of the {\em facility location} problem. Here, the objective is the sum of similarities of all the points inside a (large) collection of data points to a selected set of centers. Given a distribution over point locations, the true objective is defined as the expected value w.r.t.~this distribution, and can only be approximated as a sample average. Moreover, evaluating the function on a sample involves computation of many pairwise similarities, which is computationally prohibitive in the context of massive data sets.

In this work, we provide a formalization of such \emph{stochastic submodular maximization tasks}. 
More precisely, we consider set functions $f: 2^V \to \mathbb{R}_+$, defined as $f(S) = \mathbb{E}_{\gamma \sim \Gamma}[f_\gamma (S)]$ for $S \subseteq V$, where $\Gamma$ is an arbitrary distribution and for each realization $\gamma \sim \Gamma$, the set function $f_\gamma: 2^V \to \mathbb{R}_+$ is monotone and submodular (hence $f$ is monotone submodular). The goal is to maximize $f$ subject to some constraints (e.g. the $k$-cardinality constraint) having access only to i.i.d. samples $f_{\gamma\sim \Gamma}(\cdot)$. 

Methods for submodular maximization fall into two major categories: (i) The classic approach is to directly optimize the objective using discrete optimization methods (e.g. the \textsc{Greedy} algorithm and its accelerated variants), which are state-of-the-art algorithms (both in practice and theory), at least in the case of simple constraints, and are most widely considered in the literature; (ii) The alternative is to lift the problem into a continuous domain and exploit continuous optimization techniques available therein \citep{vondrak-pipage}. While the continuous approaches may lead to provably good results, even for more complex constraints, their high computational complexity inhibits their practicality.

In this paper we demonstrate how modern stochastic optimization techniques (such as \textsc{SGD}, \textsc{AdaGrad} \cite{adagrad} and \textsc{Adam} \cite{adam}), can be used to solve an important class of discrete optimization problems which can be modeled using weighted coverage functions. In particular, we show how to efficiently maximize them under matroid constraints by (i) lifting the problem into the continuous domain using the \emph{multilinear extension} \cite{vondrak2008optimal}, (ii) efficiently computing a concave relaxation of the multilinear extension \cite{Yaron-cover}, (iii) efficiently computing an unbiased estimate of the gradient for the concave relaxation thus enabling (projected) stochastic gradient ascent-style algorithms to maximize the concave relaxation, and (iv) rounding the resulting fractional solution without loss of approximation quality \cite{vondrak-pipage}. In addition to providing convergence and approximation guarantees, we demonstrate that our algorithms enjoy strong empirical performance, often achieving an order of magnitude speedup with less than $1\%$ error with respect to \textsc{Greedy}. As a result, the presented approach unleashes the powerful toolkit of stochastic gradient based approaches to discrete optimization problems. 

\paragraph{Our contributions.} In this paper we (i) introduce a framework for \emph{stochastic submodular optimization},
(ii) provide a general methodology for constrained maximization of stochastic submodular objectives, (iii) prove that the proposed approach guarantees a $(1-1/e)-$approximation in expectation for the class of weighted coverage functions, which is the best approximation guarantee achievable  in polynomial time unless $\mathrm{P=NP}$, (iv) highlight the practical benefit and efficiency of using continuous-based stochastic optimization techniques for submodular maximization, (v) demonstrate the practical utility of the proposed framework in an extensive experimental evaluation. We show for the first time that continuous optimization is a highly practical, scalable avenue for maximizing submodular set functions.

\section{Background and problem formulation}\label{sec:background}
Let $V$ be a ground set of $n$ elements. A set function $f:2^V\longrightarrow\RR_+$ is \emph{submodular} if for every $A,B\subseteq V$, it holds $f(A)+f(B)\geq f(A\cap B) + f(A\cup B)$. Function $f$ is said to be monotone if $f(A)\leq f(B)$ for all $A\subseteq B\subseteq V$. We focus on maximizing $f$ subject to some constraints on $S \subseteq V $. The prototypical example is maximization under the cardinality constraint, i.e., for a given integer $k$, find $S \subseteq V$, $|S|\leq k$, which maximizes $f$. Finding an exact solution for monotone submodular functions is NP-hard \cite{fiege-98}, but a $(1-1/e)$-approximation can be efficiently determined~\cite{greedyapprox}. Going beyond the $(1 -1/e)$-approximation is NP-hard for many classes of submodular functions~\cite{greedyapprox,krause2005near}. More generally, one may consider \emph{matroid constraints}, whereby $(V, \mathcal{I})$ is a matroid with the family of independent sets $\mathcal{I}$, and maximize $f$ such that $S \in \mathcal{I}$. The \textsc{Greedy} algorithm achieves a $1/2$-approximation \cite{fisher1978analysis}, but \textsc{Continuous Greedy} introduced by \citet{vondrak2008optimal, calinescu2007maximizing} can achieve a $(1-1/e)$-optimal solution in expectation. Their approach is based on the \textit{multilinear extension} of $f$, $F: [0,1]^V \to \mathbb{R}_+$, defined as 
\begin{equation} \label{multilinear} 
	F(\xb) = \sum_{S \subseteq V } f(S) \prod_{i \in S} x_i \prod_{j \notin S} (1-x_j), 
\end{equation} 
for all $\xb = (x_1,\cdots, x_n) \in [0,1]^V$. In other words, $F(\xb)$ is the expected value of of $f$ over sets wherein each element $i$ is included with probability $x_i$ independently. Then, instead of optimizing $f(S)$ over $\mathcal{I}$, we can optimize $F$ over the matroid base polytope corresponding to $(V, \Ical)$: $\Pcal = \{ \xb\in \RR^n_+\;|\; \xb(S)\leq r(S), \forall S\subseteq V, \xb(V) = r(V)\}$, where $r(\cdot)$ is the matroid's rank function. The \textsc{Continuous Greedy} algorithm then finds a solution $\xb \in \Pcal$ which provides a $(1-1/e)-$approximation. Finally, the continuous solution $\xb$ is then efficiently rounded to a feasible discrete solution without loss in objective value, using \textsc{Pipage Rounding} \cite{ageev2004pipage, calinescu2007maximizing}. The idea of converting a discrete optimization problem into a continuous one was first exploited by \citet{lovasz1983submodular} in the context of submodular minimization and this approach was recently applied to a variety of problems \cite{vondrak2007submodularity, rishabhbilmes_semidifferentials_arxiv2015, DBLP:journals/corr/abs-1010-4207}.

\paragraph{Problem formulation.} The aforementioned results are based on the \textit{oracle model}, whereby the exact value of $f(S)$ for any $S\subseteq V$ is given by an oracle. In absence of such an oracle, we face the additional challenges of {\em evaluating} $f$, both statistical and computational. In particular, consider set functions that are defined as \emph{expectations}, i.e. for $S \subseteq V$ we have
\begin{equation} \label{f-exp}
f(S) = \mathbb{E}_{\gamma \sim \Gamma}[f_\gamma (S)],
\end{equation}
where $\Gamma$ is an arbitrary distribution and for each realization $\gamma \sim \Gamma$, the set function $f_\gamma: 2^V \to \mathbb{R}$ is submodular. 
The goal is to efficiently maximize $f$ subject to constraints such as the $k$-cardinality constraint, or more generally, a matroid constraint.

As a motivating example, consider the problem of propagation of contagions through a network. The objective is to identify the most influential seed set of a given size. A propagation instance (concrete realization of a contagion) is specified by a graph $G = (V, E)$. The influence $f_G(S)$ of a set of nodes $S$ in instance $G$ is the fraction of nodes reachable from $S$ using the edges $E$. 
To handle uncertainties in the concrete realization, it is natural to introduce a probabilistic model such as the Independent Cascade \cite{kempe} model which defines a distribution $\Gcal$ over instances $G \sim \Gcal$ that share a set $V$ of nodes. The influence of a seed set $S$ is then the expectation $f(S) = \EE_{G \sim \Gcal}[f_G(S)]$, which is a monotone submodular function. Hence, estimating the expected influence is computationally demanding, as it requires summing over exponentially many functions $f_G$. Assuming $f$ as in \eqref{f-exp}, one can easily obtain an unbiased estimate of $f$ for a fixed set $S$ by random sampling according to $\Gamma$. 
The critical question is, given that the underlying function \emph{is} an expectation, can we optimize it more efficiently?

Our approach is based on continuous extensions that are linear operators on the class of set functions, namely, \textit{linear continuous extensions}. As a specific example, considering the multilinear extension, we can write $F(\xb) = \mathbb{E}_{\gamma \sim \Gamma}[F_\gamma (\xb)],$ where $F_\gamma$ denotes the extension of $f_\gamma$. As a consequence, the value of $F_\gamma(\xb)$, when $\gamma \sim \Gamma$, is an \textit{unbiased estimator} for $F(\xb)$ and unbiased estimates of the (sub)gradients may be obtained analogously. 
We explore this avenue to develop efficient algorithms for maximizing an important subclass of submodular functions that can be expressed as weighted coverage functions. Our approach harnesses a \textit{concave relaxation} detailed in Section~\ref{sec:WCF}.  

\textbf{Further related work.} The emergence of new applications, combined with a massive increase in the amount of data has created a demand for fast algorithms for submodular
optimization. A variety of approximation algorithms have been presented, ranging from submodular maximization subject to a cardinality constraint
\cite{lazier,wei2014-fast-multi-stage-icml, badanidiyuru2014fast}, submodular maximization subject to a matroid constraint \cite{calinescu2007maximizing}, non-monotone submodular maximization \cite{feige2011maximizing}, approximately submodular functions \cite{yaron-noisy}, and algorithms for submodular maximization subject to a wide variety of constraints \cite{kulik2009maximizing, feldman2011unified, vondrak2013symmetry, iyer2013submodular,  alina-stochastic}. A closely related setting to ours is online submodular maximization \cite{daniel-online}, where functions come one at a time and the goal is to provide time-dependent solutions (sets) such that a cumulative regret is minimized. In contrast, our goal is to find a single (time-independent) set that maximizes the objective \eqref{f-exp}. Another relevant setting is noisy submodular maximization,  where the evaluations returned by the oracle are noisy \cite{DBLP:journals/corr/HassidimS16, singla16noisy}. Specifically, \cite{singla16noisy} assumes a noisy but unbiased oracle (with an independent sub-Gaussian noise) which allows one to sufficiently estimate the marginal gains of items by averaging. In the context of cardinality constraints, some of these ideas can be carried to our setting by introducing additional assumptions on how the values $f_\gamma(S)$ vary w.r.t. to their expectation $f(S)$. However, we provide a different approach that does not rely on uniform convergence and compare sample and running time complexity comparison with variants of  \textsc{Greedy} in Section~\ref{sec:concave}.

\section{Stochastic Submodular Optimization}\label{section3}
We follow the general framework of~\cite{vondrak2008optimal} whereby the problem is lifted into the continuous domain, a continuous optimization algorithm is designed to maximize the transferred objective, and the resulting solution is rounded. Maximizing $f$ subject to a matroid constraint can then be done by first maximizing its multilinear extension  $F$ over the matroid base polytope and then rounding the solution. Methods such as the projected stochastic gradient ascent can be used to maximize $F$ over this polytope.

Critically, we have to assure that the computed local optima are \emph{good} in expectation. Unfortunately, the multilinear extension $F$ lacks concavity and therefore may have bad local optima. Hence, we consider \emph{concave} continuous extensions of $F$ that are \textit{efficiently computable}, and at most a constant factor away from $F$ to ensure solution quality. As a result, such a concave extension $\bar{F}$ could then be efficiently maximized over a polytope using \textit{projected stochastic gradient ascent} which would enable the application of modern continuous optimization techniques.
One class of important functions for which such an extension can be efficiently computed is the class of weighted coverage functions. 

\paragraph{The class of weighted coverage functions (WCF).}  \label{sec:WCF}
Let $U$ be a set and let $g$ be a nonnegative modular function on $U$, i.e. $g(S) = \sum_{u\in S} w(u)$, $S \subseteq U$. Let $V=\{B_1,\ldots, B_n\}$ be a collection of subsets of $U$. The \textit{weighted coverage function} $f:2^V \longrightarrow \RR^+$ defined as
\begin{align*}
 \forall S\subseteq V:\, f(S) = g\left(\textstyle\bigcup_{B_i\in S} B_i\right)
\end{align*}
is monotone submodular. For all $u\in U$, let us denote by $P_u:= \{B_i\in V\;|\; u\in B_i\}$ and by $\II (\cdot)$ the indicator function. The multilinear extension of $f$ can be expressed in a more compact way:
\begin{align}
F(\xb)&=\EE_S[f(S)]
	= \EE_S \sum_{u\in U} \II(u\in B_i\; \text{for some}\; B_i\in S)\cdot w(u)\nonumber\\
	&= \sum_{u\in U} w(u)\cdot\PP(u\in B_i\; \text{for some}\; B_i\in S)
	= \sum_{u\in U} w(u)\biggl(1 - \textstyle\prod_{B_i\in P_u}(1-x_i)\biggr) \label{compact}
\end{align}
where we used the fact that each element $B_i\in V$ was chosen with probability $x_i$.

\paragraph{Concave upper bound for weighted coverage functions.}
To efficiently compute a concave upper bound on the multilinear extension we use the framework of \citet{Yaron-cover}.
Given that all the weights $w(u)$, $u \in U$ in~\eqref{compact} are non-negative, we can construct a concave upper bound for the multilinear extension $F(\xb)$ using the following Lemma. Proofs can be found in the Appendix~\ref{appx:proof-lem-1}. 
\begin{lemma}\label{lem:upper}
For $\xb \in [0,1]^\ell$ define $\alpha(\xb) := 1 -  \prod_{i=1}^\ell (1 - x_i)$. Then the Fenchel concave biconjugate of $\alpha(\cdot)$ is $\beta(\xb) := \min\left\{1,\sum_{i=1}^\ell x_i \right\}$. Also
\[
  (1 - 1/e)\ \beta(\xb) \leq \alpha(\xb) \leq \beta(\xb)\quad \forall \xb \in [0,1]^\ell.
\]
Furthermore, $\beta$ is an extension of $\alpha$, i.e. $\forall \xb\in \{0,1\}^\ell$:  $\alpha(\xb)=\beta(\xb)$.
\end{lemma}

Consequently, given a weighted coverage function $f$ with $F(\xb)$ represented as in \eqref{compact}, we can define
\begin{align}\label{eq:fbar}
\Fbar(\xb) := \sum_{u\in U} w(u)\min\biggl\{1, \sum_{B_v\in P_u} x_v \biggr\}	
\end{align}
and conclude using Lemma~\ref{lem:upper} that $(1-1/e)\Fbar(\xb) \leq F(\xb) \leq \Fbar(\xb)$, as desired.
Furthermore, $\Fbar$ has three interesting properties: (1) It is a concave function over $[0,1]^{V}$, (2) it is equal to $f$ on vertices of the hypercube, i.e. for $\xb \in \{0,1\}^n$ one has $\Fbar(\xb) = f(\{i:x_i = 1\})$, and 
(3) it can be computed efficiently and deterministically given access to the sets $P_u$, $u \in U$. 
In other words, we can compute the value of $\Fbar(\xb)$ using at most $\mathcal{O}(|U|\times|V|)$ operations. Note that $\Fbar$ is not the \textit{tightest} concave upper bound of $F$, even though we use the tightest concave upper bounds for each term of $F$.

\paragraph{Optimizing the concave upper bound by stochastic gradient ascent.}\label{sec:concave}
Instead of maximizing $F$ over a polytope $\Pcal$, one can now attempt to maximize $\Fbar$ over $\Pcal$. Critically, this task can be done efficiently, as $\Fbar$ is concave, by using projected stochastic gradient ascent. In particular, one can control the convergence speed by choosing from the toolbox of modern continuous optimization algorithms, such as \textsc{Sgd}, \textsc{AdaGrad} and \textsc{Adam}.
Let us denote a maximizer of $\Fbar$ over $\Pcal$ by $\bar{\xb}^*$, and also a maximizer of $F$ over $\Pcal$ by $\xb^*$.  We can thus write
\begin{align*}
 F(\bar{\xb}^*) \geq (1-1/e)\Fbar(\bar{\xb}^*) \geq (1-1/e)\Fbar(\xb^*) \geq (1-1/e)F(\xb^*),
\end{align*}
which is the exact guarantee that previous methods give, and in general is the best near-optimality ratio that one can give in poly-time.
Finally, to round the continuous solution we may apply \textsc{Randomized-Pipage-Rounding} \cite{vondrak-pipage} as the quality of the approximation is preserved in expectation.
\begin{algorithm}[t]
   \caption{Stochastic Submodular Maximization via concave relaxation}
   \label{alg:scheme}
\begin{algorithmic}[1]
   \REQUIRE{matroid $\Mcal$ with base polytope $\Pcal$, $\eta_t$ (step size), $T$ (maximum \# of iterations)}
   \STATE $\xb^{(0)} \gets$  starting point in $\Pcal$
   \FOR{$t\gets0$ \textbf{to} $T-1$}
   \STATE Choose $\gb_t$ at random from a distribution such that $\EE[\gb_t|\xb^{(0)},\ldots,\xb^{(t)}] \in \partial \Fbar(\xb^{(t)})$
   \STATE $\xb^{(t+1/2)} \gets \xb^{(t)} + \eta_t\ \gb_t$
   \STATE $\xb^{(t+1)} \gets \mathrm{Project}_\Pcal(\xb^{(t+1/2)})$
   \ENDFOR
   \STATE $\bar{\xb}_T \gets \frac{1}{T}\sum_{t=1}^{T} \xb^{(t)}$
   \STATE $S \gets $ \textsc{Randomized-Pipage-Round}$(\bar{\xb}_T)$
   \STATE \textbf{return} $S$ such that $S \in \Mcal$, $\EE[f(S)]\geq(1-1/e)f(OPT) - \varepsilon(T)$.
\end{algorithmic}
\end{algorithm}

\paragraph{Matroid constraints.} Constrained optimization can be efficiently performed by projected gradient ascent whereby after each step of the stochastic ascent, we need to project the solution back onto the feasible set. For the case of matroid constraints, it is sufficient to consider projection onto the matroid base polytope. This problem of projecting on the base polytope has been widely studied and fast algorithms exist in many cases \cite{bach-tutorial, brucker1984n, pardalos1990algorithm}. While these projection algorithms were used as a key subprocedure in constrained submodular minimization, here we consider them for submodular maximization. Details of a fast projection algorithm for the problems considered in this work are presented the Appendix~\ref{appx:proj}.
Algorithm~\ref{alg:scheme} summarizes all steps required to maximize $f$ subject to matroid constraints.

\paragraph{Convergence rate.}
Since we are maximizing a concave function $\Fbar(\cdot)$ over a matroid base polytope $\Pcal$, convergence rate (and hence running time) depends on $B := \max_{\xb \in \Pcal} ||\xb||$, as well as maximum gradient norm $\rho$ (i.e.  $ || \gb_t || \leq \rho$ with probability $1$).~\footnote{Note that the function $\Fbar$ is neither smooth nor strongly concave as functions such as $\min\{1,x\}$ are not smooth or strongly concave.} In the case of the base polytope for a matroid of rank $r$, $B$ is $\sqrt{r}$, since each vertex of the polytope has exactly $r$ ones. 
Also, from \eqref{eq:fbar}, one can build a rough upper bound for the norm of the gradient:
\[
	||\gb|| \leq ||\textstyle\sum_{u\in U} w(u)\one_{P_u}|| \leq \bigl(\displaystyle{\max_{u\in U}}|P_u| \bigr)^{1/2}  \displaystyle{\sum_{u\in U}}w(u),
\]
which depends on the weights $w(u)$ as well as $|P_u|$ and is hence problem-dependent. We will provide tighter upper bounds for gradient norm in our specific examples in the later sections. With $\eta_t = B/\rho\sqrt{t}$, and classic results for \textsc{SGD} \cite{shalev}, we have that
\[
  \Fbar(\xb^*) - \EE[\Fbar(\bar{\xb}_T)] \leq B \rho/\sqrt{T},
\]
where $T$ is the total number of SGD iterations and $\bar{\xb}_T$ is the final outcome of SGD (see Algorithm~\ref{alg:scheme}). Therefore, for a given $\varepsilon>0$, after $T\geq B^2\rho^2/\varepsilon^2$ iterations,  we have 
\[
  \Fbar(\xb^*) - \EE[\Fbar(\bar{\xb}_T)] \leq \varepsilon. 
\]
Summing up, we will have the following theorem:

\begin{theorem}\label{main2}
	Let $f$ be a weighted coverage function, $\Pcal$ be the base polytope of a matroid $\Mcal$, and $\rho$ and $B$ be as above. Then for each $\epsilon>0$, Algorithm~\ref{alg:scheme} after $T = B^2\rho^2/\varepsilon^2$ iterations, produces a set $S^*\in\Mcal$ such that $\EE [f(S^*)] \geq (1-1/e)\max_{S \in \Mcal}f(S) - \varepsilon$. 
\end{theorem}
\paragraph{Remark.}
	Indeed this approximation ratio is the best ratio one can achieve, unless P$=$NP \cite{fiege-98}. A key point to make here is that our approach also works for more general constraints (in particular is efficient for \emph{simple} matroids such as partition matroids). In the latter case, \textsc{Greedy} only gives $\frac{1}{2}$-approximation and fast discrete methods like \textsc{Stochastic-Greedy} \cite{lazier} do not apply, whereas our method still yields an $(1-1/e)$-optimal solution.

\paragraph{Time Complexity.} One can compute an upper bound for the running time of Algorithm~\ref{alg:scheme} by estimating the time required to perform gradient computations, projection on $\Pcal$, and rounding. For the case of uniform matroids, projection and rounding take $\mathcal{O}(n\log n)$ and $\mathcal{O}(n)$ time, respectively (see Appendix~\ref{appx:proj}). Furthermore, for the applications considered in this work, namely expected influence maximization and exemplar-based clustering, we provide linear time algorithms to compute the gradients. Also when our matroid is the $k$-uniform matroid (i.e. $k$-cardinality constraint), we have $B = \sqrt{k}$. By Theorem~\ref{main2}, the total computational complexity of our algorithm is 
$\mathcal{O}(\rho^2 kn(\log n)/\varepsilon^2)$.

\paragraph{Comparison to \textsc{Greedy}.} Let us relate our results to the classical approach. When running the \textsc{Greedy} algorithm in the stochastic setting, one estimates $\hat{f}(S) := \frac{1}{s}\sum_{i=1}^s f_{\gamma_i}(S)$ where $\gamma_1, \ldots, \gamma_s$ are i.i.d. samples from $\Gamma$. The following proposition bounds the sample and computational complexity of \textsc{Greedy}. The proof is detailed in the Appendix~\ref{appx:hoff}.
\begin{proposition}\label{prop:iidsamples}
  Let $f$ be a submodular function defined as~\eqref{f-exp}. Suppose $0\leq f_\gamma(S) \leq H$ for all $S\subseteq V$ and all $\gamma\sim \Gamma$. 
  Assume $S^*$ denotes the optimal solution for $f$ subject to $k$-cardinality constraint and $S_k$ denotes the solution computed by the greedy algorithm on $\fhat$ after $k$ steps. 
  Then, in order to guarantee
  \[
    \PP[f(S_k) \geq (1 - 1/e)f(S^*) - \varepsilon] \geq 1-\delta,
  \]
  it is enough to have 
  \[
    s \in \Omega\biggl(H^2 ( k\log n + \log(1/\delta) ) /\varepsilon^2\biggr),
  \]
  i.i.d. samples from $\Gamma$. 
  The running time of \textsc{Greedy} is then bounded by
  \[
    \mathcal{O}\biggr( \tau H^2 nk  ( k \log n + \log(1/\delta) )/\varepsilon^2 \biggr),
  \]
  where $\tau$ is an upper bound on the computation time for a single evaluation of $f_\gamma(S)$.
\end{proposition}
As an example, let us compare the worst-case complexity bound obtained for  SGD (i.e. $\mathcal{O}(\rho^2 kn(\log n)/\varepsilon^2)$)  with that of \textsc{Greedy}  for the influence maximization problem.  Each single function evaluation for \textsc{Greedy} amounts to computing the total influence of a set in a sample graph, which makes $\tau = O(n)$ (here we assume our sample graphs satisfy $|E|=O(|V|)$). Also, a crude upper bound for the size of the gradient for each sample function is $H\sqrt{n}$ (see Appendix~\ref{sec:det-inf}). Hence, we can deduce that SGD can have a factor $k$ speedup w.r.t. to \textsc{Greedy}. 

\section{Applications}
We will now show how to instantiate the \emph{stochastic submodular maximization framework} using several prototypical discrete optimization problems.
\vspace{-.3cm}
\paragraph{Influence maximization.}
As discussed in Section~\ref{sec:background}, the Independent Cascade \cite{kempe} model defines a distribution $\Gcal$ over instances $G \sim \Gcal$ that share a set $V$ of nodes. The influence $f_G(S)$ of a set of nodes $S$ in instance $G$ is the fraction of nodes reachable from $S$ using the edges $E(G)$. The following Lemma shows that the influence belongs to the class of WCF.
\begin{lemma}\label{lem:infl-wcf}
	The influence function $f_G(\cdot)$ is a WCF. Moreover,
\begin{align}
&F_G(\xb) = \EE_S[f_G(S)] = \frac{1}{|V|}\sum_{v\in V} (1 - \textstyle\prod_{u\in P_v}(1-x_u))\\
&\Fbar_G(\xb) = \frac{1}{|V|}\sum_{v\in V}\min\{1, \textstyle\sum_{u\in P_v}x_u\},
\end{align}
where $P_v$ is the set of all nodes having a (directed) path to $v$. 
\end{lemma}

We return to the problem of maximizing $f_{\Gcal}(S) = \EE_{G \sim \Gcal}[f_G(S)]$ given a distribution over graphs $\Gcal$ sharing nodes $V$. Since $f_{\Gcal}$ is a weighted sum of submodular functions, it is submodular. Moreover,
\begin{align*}
F(\xb) &= \EE_S[f_\Gcal(S)] = \EE_S[\EE_G[f_G(S)]] = \EE_G[\EE_S[f_G(S)]] = \EE_G[F_G(\xb)]\\
	&= \EE_G\left[ \frac{1}{|V|}\sum_{v\in V} (1 - \textstyle\prod_{u\in P_v}(1-x_u))\right].
\end{align*}
Let $\Ucal$ be the uniform distribution over vertices. Then,
\begin{align}
F(\xb) &= \EE_G\biggl [ \tfrac{1}{|V|}\textstyle\sum_{v\in V} (1 - \textstyle\prod_{u\in P_v}(1-x_u))\biggr]  = \EE_G \biggl[\ \EE_{v\sim \Ucal}\  [ 1 - \textstyle\prod_{u\in P_v}(1-x_u) ] \biggr] \label{eq:multim_inf},
\end{align} 
and the corresponding upper bound would be
\begin{align}\label{eq:ubound_inf}
\Fbar(\xb) = \EE_G \biggl[ \ \EE_{v\sim \Ucal}\bigl[ \min\{1, \textstyle\sum_{u\in P_v}x_u\}\bigr] \biggr].
\end{align}
This formulation proves to be helpful in \textit{efficient} calculation of subgradients, as one can obtain a random subgradient in linear time. For more details see Appendix~\ref{sec:det-inf}. We also provide a more efficient, \textit{biased} estimator of the expectation in the Appendix.
\vspace{-.3cm}
\paragraph{Facility location.} 
Let $G=(X\dot\cup Y, E)$ be a complete weighted bipartite graph with parts $X$ and $Y$ and nonnegative weights $w_{x,y}$. The weights can be considered as utilities or some similarity metric. We select a subset $S\subseteq X$ and each $y\in Y$ selects $s\in S$ with the highest weight $w_{s,y}$. Our goal is to maximize the average weight of these selected edges, i.e. to maximize
\begin{align}\label{eq:facilityloc}
	f(S) = \frac{1}{|Y|}\sum_{y\in Y} \max_{s\in S} w_{s, y}
\end{align}
given some constraints on $S$. This problem is indeed the \textit{Facility Location} problem, if one takes $X$ to be the set of facilities and $Y$ to be the set of customers and $w_{x,y}$ to be the utility of facility $x$ for customer $y$. Another interesting instance is the \textit{Exemplar-based Clustering} problem, in which $X= Y$ is a set of objects and $w_{x,y}$ is the similarity (or inverted distance) between objects $x$ and $y$, and one tries to find a subset $S$ of exemplars (i.e. \textit{centroids}) for these objects.

The stochastic nature of this problem is revealed when one writes \eqref{eq:facilityloc} as the expectation $f(S) = \EE_{y \sim \Gamma} [f_y(S)]$, where $\Gamma$ is the uniform distribution over $Y$ and $f_y(S) := \max_{s\in S} w_{s,y}$. One can also consider this more general case, where $y$'s are drawn from an unknown distribution, and one tries to maximize the aforementioned expectation.

First, we claim that $f_y(\cdot)$ for each $y \in Y$ is again a weighted coverage function. For simplicity, let $X=\{1,\ldots, n\}$ and set $m_i \doteq w_{i, y}$, with $m_1\geq \cdots \geq m_n$ and $m_{n+1} \doteq 0$.
\begin{lemma}\label{lem:facil-wcf}
	The utility function $f_y(\cdot)$ is a WCF. Moreover,
	\begin{align}
  &\textstyle F_y(\xb) = \sum_{i=1}^n (m_i - m_{i+1})(1-\prod_{j=1}^i (1-x_j)),\label{eq:fac-mult}\\
  &\textstyle \Fbar_y(\xb) = \sum_{i=1}^n (m_i - m_{i+1})\min\{1, \sum_{j=1}^i x_j\}. \label{kmeans-upper}
\end{align}
\end{lemma}

We remark that the gradient of both $F_y$ and $\Fbar_y$ can be computed in linear time using a recursive procedure. We refer to Appendix~\ref{sec:det-exe} for more details.

\section{Experimental Results}\label{sec:exe_exp}

We demonstrate the practical utility of the proposed framework and compare it to standard baselines. We compare the performance of the algorithms in terms of their wall-clock running time and the obtained utility. We consider the following problems:

\begin{figure}[ht]
  \centering
  \includegraphics[width=.40\textwidth]{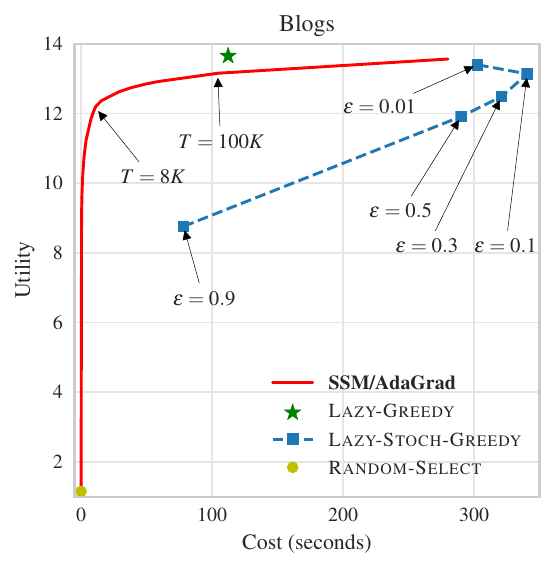}
  \includegraphics[width=.40\textwidth]{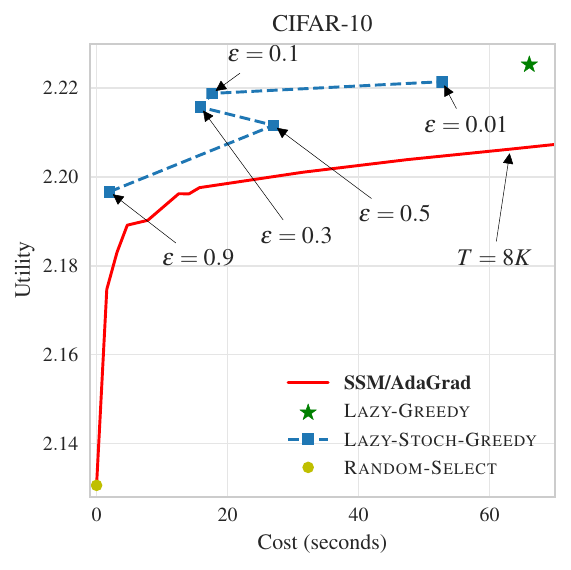}\\
  \includegraphics[width=.40\textwidth]{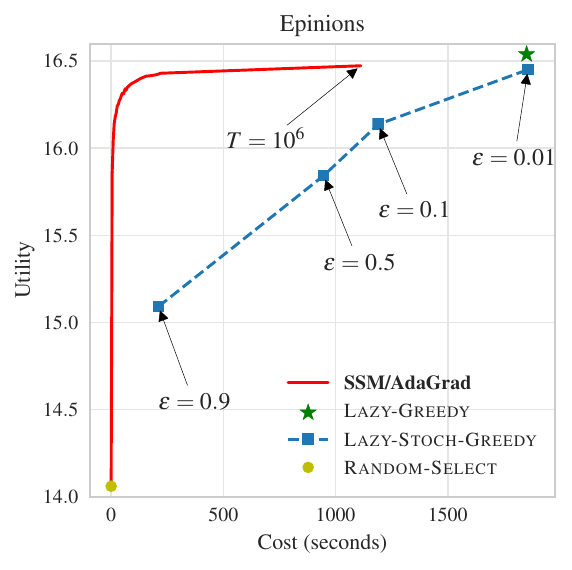}
  \includegraphics[width=.40\textwidth]{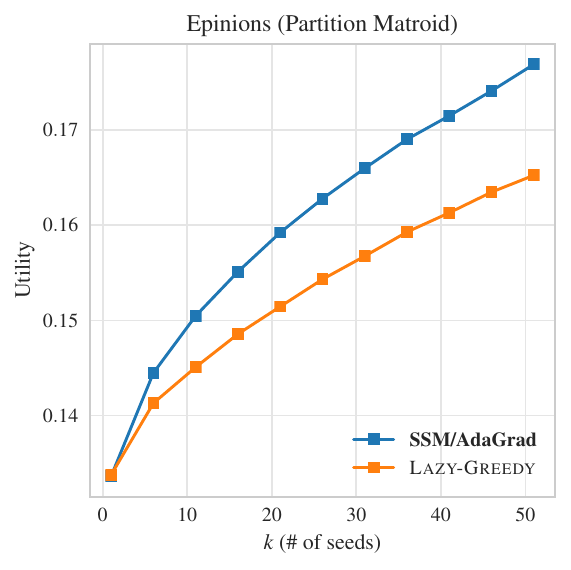}
  \caption{In the case of Facility location for Blog selection as well as on influence maximization on Epinions, the proposed approach reaches the same utility \emph{significantly} faster. On the exemplar-based clustering of CIFAR, the proposed approach is outperformed by \textsc{Stochastic-Greedy}, but nevertheless reaches $98.4\%$ of the \textsc{Greedy} utility in a few seconds (after less than $1000$ iterations). On Influence Maximization over \textit{partition matroids}, the proposed approach significantly outperforms \textsc{Greedy}. \label{fig:results}}
\end{figure}

\begin{itemize}[leftmargin=*]
\item \textbf{Influence Maximization for the Epinions network\footnote{\texttt{http://snap.stanford.edu/}}.} The network consists of 75\,879 nodes and 508\,837 directed edges. We consider the subgraph induced by the top 10\,000 nodes with the largest out-degree and use the independent cascade model \cite{kempe}. The diffusion model is specified by a fixed probability for each node to influence its neighbors in the underlying graph. We set this probability $p$ to be $0.02$, and chose the number of seeds $k=50$.

\item \textbf{Facility Location for Blog Selection.} We use the data set used in \cite{glance2005deriving}, consisting of 45\,193 blogs, and 16\,551 cascades. The goal is to detect information cascades/stories spreading over the blogosphere. This dataset is \textit{heavy-tailed}, hence a small random sample of the events has high variance in terms of the cascade sizes. We set $k=100$.

\item \textbf{Exemplar-based Clustering on \textsc{CIFAR}-10.} The data set contains 60\,000 color images with resolution $32\times 32$. We use a single batch of 10\,000 images and compare our algorithms to variants of \textsc{Greedy} over the full data set. We use the Euclidean norm as the distance function and set $k=50$. Further details about preprocessing of the data as well as formulation of the submodular function can be found in Appendix~\ref{appx:exem}.
\end{itemize}

\paragraph{Baselines.}
In the case of cardinality constraints, we compare our stochastic continuous optimization approach against the most efficient discrete approaches (\textsc{Lazy}-)\textsc{Greedy} and (\textsc{Lazy}-)\textsc{Stochastic-Greedy}, which both provide optimal approximation guarantees. For \textsc{Stochastic-Greedy}, we vary the parameter $\varepsilon$ in order to explore the running time/utility tradeoff.  We also report the performance of randomly selected sets. For the two facility location problems, when applying the greedy variants we can evaluate the exact objective (true expectation). In the Influence Maximization application, computing the exact expectation is intractable. Hence, we use an empirical average of $s$ samples (cascades) from the model. We note that the number of samples suggested by Proposition~\ref{prop:iidsamples} is overly conservative, and instead we make a practical choice of $s=10^3$ samples.

\paragraph{Results.}
The results are summarized in Figure~\ref{fig:results}. On the blog selection and influence maximization applications, the proposed continuous optimization approach outperforms \textsc{Stochastic-Greedy} in terms of the running time/utility tradeoff. In particular, for blog selection we can compute a solution with the same utility $26\times$ faster than \textsc{Stochastic-Greedy} with $\varepsilon=0.5$. Similarly, for influence maximization on Epinions we the solution $88\times$ faster than \textsc{Stochastic-Greedy} with $\varepsilon=0.1$. On the exemplar-based clustering application \textsc{Stochastic-Greedy} outperforms the proposed approach. We note that the proposed approach is still competitive as it recovers $98.4\%$ of the value after less than thousand iterations.

We also include an experiment on Influence Maximization over \textit{partition matroids} for the Epinions network. In this case, \textsc{Greedy} only provides a $1/2$ approximation guarantee and \textsc{Stochastic-Greedy} does not apply. To create the partition, we first sorted all the vertices by their out-degree.  Using this order on the vertices, we divided the vertices into two partitions, one containing vertices with even positions, other containing the rest. Figure~\ref{fig:results} clearly demonstrates that the proposed approach outperforms \textsc{Greedy} in terms of utility (as well as running time).

\paragraph{Acknowledgments}
The research was partially supported by ERC StG 307036. We would like to thank Yaron Singer for helpful comments and suggestions. 

\bibliographystyle{plainnat}
\bibliography{stochastic}

\newpage
\appendix
\section{Proof of Lemma \ref{lem:upper}}\label{appx:proof-lem-1}
Here we prove the inequality mentioned in the lemma. Proof of the fact of being Fenchel biconjugate is in  Appendix~\ref{appx:fenchel}.

We prove the left-hand-side inequality, since the right-hand-side inequality is a consequence of Fenchel biconjugate-ness. 

Let $\theta := \sum_{i=1}^\ell x_i$. We note from the inequality $1 - x \leq \exp{(-x)}$ that $\prod_{i=1}^\ell (1 - x_i) \leq \exp{(-\theta)}$. We thus obtain
$$1 - \prod_{i=1}^\ell (1 - x_i) \geq 1 - \exp{(-\theta)}.$$
Now, if $\theta \geq 1$ then the result is clear. Also, if $\theta < 1$, then we note that the function $(1-\exp{(-\theta)}) / \theta $ is decreasing for $\theta \in (0,1)$, and hence, $1-\exp{(-\theta)} \geq \theta (1 - 1/e)$. The left-hand-side inequality thus follows immediately.

\section{Proof of Proposition \ref{prop:iidsamples}}\label{appx:hoff}
Note that the total number of subsets of cardinality less than $k$ is bounded from above by $k \binom{n}{k}$.  For each such set $S$ we want the estimate $\hat{f}(S) := \frac{1}{s}\sum_{i=1}^s f_{\gamma_i}(S)$ to be at most $\epsilon$ away from $f(S)$.  Also, note that the function $\hat{f}$ is itself a submodular function and maximizing it would give a $(1-1/e)$-approximation to its optimum. Hence, it is enough to have enough samples such that for all subsets $S$ of cardinality at most $k$ the two values $f(S)$ and $\hat{f}(S)$ differ by at most epsilon.  By using Hoeffding's inequality and a union bound over all the subsets of cardinality at most $k$ (note that $\log (n \binom{n}{k}) = \Ocal(k \log (n))$) we get the result.

\section{Proof of Lemmas \ref{lem:infl-wcf} and \ref{lem:facil-wcf}}

\subsection{Lemma \ref{lem:infl-wcf}}
\begin{proof}
Let $A := \{ C_v\;|\; v \in V \}$, where $C_v$ is the set of vertices reachable from $v$. By construction, there is a one-to-one correspondence between elements of $A$ and $V$, namely $C_v\leftrightarrow v$. For $T\subseteq A$, let $S\subseteq V$ be its corresponding subset in $V$, i.e. 
$S = \{v\in V\;|\; v\leftrightarrow C_v, C_v \in T\}$. It's obvious that $\bigcup_{v\in S} C_v = \bigcup_{C_v\in T} C_v$. Setting $g(T) = \frac{|T|}{|V|}$, makes $f'_G(T) := g(\bigcup_{C_v\in T} C_v)$ a WCF. But $f'_G(T) = f_G(S)$, so $f_G(\cdot)$ is also a WCF.

Moreover, for each $v \in V$, the set $P_v$ is the set of all elements of $A$ that contain $v$, which are precisely those vertices from which there is a (directed) path to $v$. We also relax our notation, and replace any element of $A$ by its correspondent in $V$. Hence, 
\begin{align*}
&F_G(\xb) = \EE_S[f_G(S)] = \frac{1}{|V|}\sum_{v\in V} (1 - \textstyle\prod_{u\in P_v}(1-x_u))\\
&\Fbar_G(\xb) = \frac{1}{|V|}\sum_{v\in V}\min\{1, \textstyle\sum_{u\in P_v}x_u\},
\end{align*}
 which are poly-time computable since one can find $P_v$ with a simple BFS algorithm in $\Ocal(|V|+|E|)$ for each $v\in V$.

\end{proof}

\subsection{Lemma \ref{lem:facil-wcf}}
\begin{proof}
Write $f(\cdot)$ instead of $f_y(\cdot)$.

  Let $V = \{C_i\;|\; 1\leq i\leq n\}$, where
$C_i = \{i,\ldots, n\}$, and let $w(i) = m_i - m_{i+1}$ (set $m_{n+1}=0$). Note that there is a natural bijection
between $V$ and $U$, namely $C_i\leftrightarrow i$. Let $g$ be the modular function with weights $w(i)$, defined on $2^U$, and define the WCF $f':2^V\rightarrow \RR_+$ as  
\begin{align}\label{eq:exe}
  f'(S) := \textstyle g(\bigcup_{i\in S} C_i) = \sum_{j\in \bigcup_{i\in S} C_i} w(j).
\end{align}
Since $C_i$'s are forming a decreasing chain, $\bigcup_{i\in S} C_i = C_{\min S}$ and \eqref{eq:exe} becomes
$$f'(S) = \sum_{j\in C_{\min S}} w(j) = \sum_{j = \min S}^n w(j) = m_{\min S} - m_{n+1} = \max_{i\in S} m_i,$$
which is exactly $f(S)$. 

Furthermore, $P_i$ is simply the set $\{1, \ldots, i\}$.
Hence, we can write the multilinear extension and the corresponding upper bound as
\begin{align*}
  &\textstyle F_y(\xb) = \sum_{i=1}^n (m_i - m_{i+1})(1-\prod_{j=1}^i (1-x_j)),\\
  &\textstyle \Fbar_y(\xb) = \sum_{i=1}^n (m_i - m_{i+1})\min\{1, \sum_{j=1}^i x_j\}. 
\end{align*}
\end{proof}

\section{Fast Algorithms for Projection and Rounding}\label{appx:proj}
In this section, we show how projection (w.r.t. Mahalanobis norm) can be done in time $O(n \log n)$ and rounding in time $O(n)$ for the uniform matroid. This projection algorithm also proves to be useful in case of partition matroid polytope. We also discuss a projection method on general matroid base polytopes, based on the method of \citet{active-set-bach}, which needs to solve a total number of $n$ submodular function minimization (SFM) tasks (details below).

\subsection{Efficient projection on the uniform matroid}
Let $G$ be a diagonal matrix with positive entries, $G = \mathrm{diag}(g_1, \ldots, g_n)$. Our aim is to project a vector $\yb \in \RR_{+}^n$ on the uniform matroid base polytope defined as
\[
  P_k = \{\xb\in\RR^n_+\;|\; \textstyle\sum x_i = k,\; 0\leq x_i\leq 1\}.
\]
The polytope $P_k$ is the convex hull of all the vectors that have precisely $k$ ones and $n-k$ zeros.
Projecting $\yb$ onto $P_k$  entails finding a point $\xb$ in $P_k$, such that
\[
  \xb = \argmin_{\xb \in P_k}\tgnorm{\xb - \yb}^2 := \argmin_{\xb \in P_k}\,(\xb-\yb)^\top G\,(\xb-\yb),
\]
where $\tgnorm{\cdot}$ is the Mahalanobis norm (i.e. the Mahalanobis distance to $\zero$). Note that in the special case of $G = I$, this problem boils down to orthogonal projection of $\yb$ onto $P_k$. We first transform this problem into an orthogonal projection, and solve that projection in $O(n\log n)$.
\begin{align}
	\xb & = \argmin_{\xb \in P_k}\,(\xb-\yb)^\top G\,(\xb-\yb) \nonumber \\
		& = \argmin_{\xb \in P_k} \tnorm{\ub-\wb}^2,\quad\text{where}\; \ub = G^{1/2}\xb\;\text{and}
		\; \wb = G^{1/2}\yb \nonumber \\
		& = G^{-1/2} \argmin_{\ub \in G^{1/2}P_k} \tnorm{\ub-\wb}^2 \label{eq:ortho-version},
\end{align}
where \eqref{eq:ortho-version} suggests an orthogonal projection on the polytope $G^{1/2}P_k$. By defining the vector $\cb = (g_1^{-1/2}, \ldots, g_n^{-1/2})$, one has $G^{1/2}P_k = \{ \xb\in\RR^n_+\;|\; \cb^\top \xb = k, 0 \leq x_i \leq \tfrac{1}{c_i} \}$. Theorem~\ref{thm:proj} shows that this projection can be done in $O(n\log n)$, and Algorithm~\ref{alg:proj} depicts the algorithm achieving the solution.
\begin{theorem}\label{thm:proj}
	Let $P = \{ \xb\in\RR^n_+\;|\; \cb^\top \xb = k, 0 \leq x_i \leq \tfrac{1}{c_i} \}$, where $\cb\in \RR^n_+$ is given. Then for any given point $\yb\in\RR^n_+$ one can find the solution to $\argmin_{\xb\in P}\frac{1}{2}\tnorm{\xb - \yb}^2$ in $O(n\log n)$ time. Moreover this solution is unique.
\end{theorem}
\begin{proof}
	Let us begin by writing the KKT optimality conditions for the projected vector $\xb$. The Lagrangian is defined by
	$$\mathcal{L}(\xb, \alpha, \bbeta, \bgamma) = \tfrac{1}{2}\tnorm{\xb - \yb}^2 + \alpha(\cb^\top \xb - k) - \bbeta^\top \xb + \bgamma^\top(\xb - 1/\cb),$$
	where $\alpha\in\RR$ and $\bbeta,\bgamma\in\RR^n_+$. Minimizing the Lagrangian w.r.t. $\xb$ gives for each $i\in [n]$:
	\begin{align}
		x_i = y_i - \alpha c_i + \beta_i - \gamma_i,
	\end{align}
	and also considering complementary slackness, we should have $\beta_i x_i = 0$ and $\gamma_i(x_i - 1/c_i) = 0$.
	If one provides suitable $\xb$ and $\alpha, \bbeta, \bgamma$ that satisfy the equations above, then $\xb$ would be the optimal solution. In what follows, we construct $\xb$ and provide suitable $\alpha, \bbeta, \bgamma$.
	
	For each $\alpha\in \RR$, define $\xb(\alpha) := \min\{\tfrac{1}{\cb} , \max \{0, \yb - \alpha \cb\}\}$, where $\min$ and $\max$ are applied element-wise. By definition, one has $\zero \leq \xb(\alpha) \leq \tfrac{1}{\cb}$. Let $h(\alpha) := \cb^\top \xb(\alpha)$. We claim that if for a value of $\alpha$, $h(\alpha) = k$, we are done, since $\xb(\alpha) \in P$, and it satisfies the KKT conditions: If $x(\alpha)_i = 0$, by definition of $\xb(\alpha)$ it means that $y_i - \alpha c_i \leq 0$, so we can set $\beta_i = -(y_i - \alpha c_i) \geq 0$ and $\gamma_i = 0$. If $x(\alpha)_i = \frac{1}{c_i}$, it means $y_i - \alpha c_i \geq \frac{1}{c_i}$, so we can set $\beta_i = 0$ and $\gamma_i = y_i - \alpha c_i - \frac{1}{c_i} \geq 0$. Otherwise, $0<x(\alpha)_i<\frac{1}{c_i}$, which in that case we set $\beta_i = \gamma_i = 0$.
	
	So it suffices to provide an $\alpha$ such that $h(\alpha) = k$. For each $i\in [n]$, define $\ubar{\alpha}_i := \frac{y_ic_i - 1}{c_i^2}$ and $\bar{\alpha}_i := \frac{y_i}{c_i}$. It's obvious that if $\alpha \leq \ubar{\alpha}_i$ then $x(\alpha)_i = \frac{1}{c_i}$, if $\alpha \geq \bar{\alpha}_i$ then $x(\alpha)_i = 0$, and otherwise $x(\alpha)_i = y_i - \alpha c_i$. So $x(\alpha)_i$ is a continuous decreasing function, and so will be $h(\alpha)$. Note that if $\alpha \leq \min\{\ubar{\alpha}_i\}$, then $h(\alpha) = n$ and if $\alpha \geq \max\{\ubar{\alpha}_i\}$, then $h(\alpha) = 0$. So by continuity, there is some $\alpha^*$ such that $h(\alpha^*) =k$. Now let $\alpha_1 < \ldots < \alpha_s$ be the set of all distinct values among $\ubar{\alpha}_i$ and $\bar{\alpha}_i$. It's clear that for all $\alpha \in [\alpha_i, \alpha_{i+1}]$, $h(\cdot)$ is a linear function. By exploiting this fact, we can find $\alpha^*$ by searching through these endpoints. Detailed procedure is explained in Algorithm~\ref{alg:proj}.
\end{proof}

\begin{algorithm}[t]
   \caption{Projection on the Scaled Uniform Matroid Polytope}
   \label{alg:proj}
\begin{algorithmic}[1]
\STATE \textbf{Input:} vectors $\yb, \cb \in \RR_+^n$ and $k\in\NN$, s.t. $k \leq n$.
\STATE $\ubar{\alpha}_i \leftarrow \frac{y_ic_i - 1}{c_i^2}, \bar{\alpha}_i \leftarrow \frac{y_i}{c_i}, \forall i\in [n]$
\STATE $S \leftarrow \{ \ubar{\alpha}_i \} \cup \{\bar{\alpha}_i \}$
\STATE Sort elements in $S$, so that $S = \{\alpha_1 < \ldots < \alpha_s \}$
\STATE $h \leftarrow n, \quad \alpha \leftarrow \min S - 1, \quad m \leftarrow 0$
\FOR{$i\in [s]$}
	\STATE  $h' \leftarrow h + (\alpha_i - \alpha)m$  \COMMENT{calculate function value at the new point using current slope $m$}\\
	\COMMENT{check if $\alpha^*$ is between $\alpha_i$ and $\alpha_{i-1}$}
	\IF{$h' < k \leq h$}
		\STATE $\alpha^* \leftarrow (\alpha_i - \alpha) \frac{h - k}{h - h'} + \alpha$
		\STATE return the projected vector $\xb$ as follows:
		$$
		x_j = \left\{
		\begin{array}{ll} 
			1/c_j & \alpha^* < \ubar{\alpha}_j \\
			y_j - \alpha^* c_j & \ubar{\alpha}_j \leq \alpha^* \leq \bar{\alpha}_j \\
			0 & \bar{\alpha}_j < \alpha^* 
		\end{array} \right.
		$$
	\ENDIF
	\STATE $m \leftarrow m - \sum_{j:\ubar{\alpha}_j = \alpha} c_j^2$ \COMMENT{for these $j$, $x(\alpha)_j$'s slope is changing from $0$ to $-c_j$}
	\STATE $m \leftarrow m + \sum_{j:\bar{\alpha}_j = \alpha} c_j^2$ \COMMENT{for these $j$, $x(\alpha)_j$'s slope is changing from $-c_j$ to $0$}
	\STATE $h \leftarrow h', \quad \alpha \leftarrow \alpha_i$
\ENDFOR
\end{algorithmic}
\end{algorithm}

\subsection{Efficient projection on Partition matroid base polytope}
Let $V$ be a ground set and $A_1,\ldots, A_m$ be a partition of $V$. A \textit{partition matroid}, includes all sets $S\subseteq V$ such that for all $i\in[m]$ we have $|A_i \cap S| \leq k$. It's easy to see that the base polytope would be
\[
	\Pcal = \left\{\xb\in[0,1]^V\;|\; \forall i \in [m]: \sum_{j\in A_i} x_j = k\right\}.
\]
 In order to project onto $\Pcal$, we first note that it becomes a separable objective, partitioned over $A_i$. This means that it is sufficient to project $\yb|_{A_i}$ onto the uniform matroid of $A_i$, for all $i\in[m]$. Since each projection takes $\Ocal(|A_i|\log |A_i|)$ time, the total process would be $\Ocal(n\log n)$.

\subsection{Projection on general matroid base polytopes} 
Let us now ask whether there is an efficient projection algorithm for general matroid polytopes. Here, we argue that the method proposed by \citet{active-set-bach} would be a reasonable candidate in the case of general matroid polytopes.

Let $g:2^V\rightarrow \RR_+$ be a submodular function, such that $g(\emptyset) = 0$, and let $g_L: \RR^n \to \mathbb{R}_+$ be its Lovasz extension. 
We define the \textit{base polytope} of $g$ as the set 
$$\mathcal{B} = \{\sb\in\RR^n\;|\; \sb(V)=g(V), \forall A\subset V: \sb(A) \leq g(A)\}.$$
It can be shown \cite{bach-tutorial} that  the Lovasz extension is the \textit{support function} of this polytope, i.e.
 \begin{equation} \label{lovasz-support}
 g_L(\xb) = \sup_{\sb\in \mathcal{B}} \sb^\top \xb.
 \end{equation}
For any $\yb\in \RR^n$ consider the task of minimizing  the following objective with respect to $\xb\in \RR^n$:
\begin{align}\label{eq:tv}g_L(\xb) - \yb^\top \xb + \tfrac{1}{2}\tnorm{\xb}^2.\end{align}
By using \eqref{lovasz-support}, we can rewrite \eqref{eq:tv} in the following dual form
\begin{align} 
\label{eq:dualtv}\min_{\xb\in\RR^n} g_L(\xb) -\yb^\top \xb + \tfrac{1}{2}\tnorm{\xb}^2 = \max_{\sb\in \mathcal{B}} -\tfrac{1}{2}\tnorm{\sb-\yb}^2,
\end{align}
in which the latter expression is precisely the projection of $\yb$ on $\mathcal{B}$. 
In \citet{active-set-bach}, the authors have exploited the structural properties of the Lovasz extension and the faces of the base polytope to create the so-called ``Active-set" algorithm. The Active-set algorithm iteratively solves instances of isotonic regression as well as submodular function minimization tasks, whose overall complexity is less than a single submodular function minimization call‌ (recall that by submodular function minimization, we mean the task of solving $\min_{\xb\in [0,1]^n} (g_L(\xb) - \yb^\top \xb)$). By knowing \eqref{eq:dualtv}, the algorithm can be viewed as a sequence of iterative projections on outer-approximations of the base polytope.

For any matroid, its associated rank function is a monotone submodular function. Also, the base polytope for a matroid's rank function is exactly the matroid base polytope. As a result of \eqref{eq:dualtv},  we can use the Active-set algorithm to perform projections on the matroid base polytope. Interestingly, in the case of uniform matroids, the main parts of our projection scheme has similar counterparts as in the Active-set scheme. However, runtime complexity is significantly different due to several differences such as optimality checks: In our approach, this check is done in $\Ocal(1)$, but in Active-set scheme, in each iteration, one should solve approximately $\Ocal(n)$ submodular minimization tasks. However, the Active-set approach is more general, as explained above.

\subsection{The \textsc{Randomized-Pipage-Rounding} procedure}
The randomized pipage rounding procedure was first proposed in \cite{vondrak-pipage} for any matroid $\Mcal$. 
Here, we show how this procedure can be efficiently done (in linear time) for the uniform matroid. 
Suppose we have a matroid $\Mcal$ and a point $\yb := (y_1, \cdots, y_n)$ in its corresponding base polytope. We want
to round $\yb$ to a vertex of the base polytope. In each step of the algorithm,
one has a fractional solution $\yb$ and a tight set $T$ containing at least two fractional variables (recall that if the
matroid rank function is $r(\cdot)$, a set $T$ is tight if $\yb(T) = r(T)$; Tight
sets are exactly those constraints in the base polytope who are tight at $y$). It modifies two fractional
variables in such a way that their sum remains constant, until some variable
becomes integral or a new constraint becomes tight. Note that since the sum of
all of elements of $\yb$ is an integer (rank of the matroid), there exist at least two
fractional variables in the case that the point is fractional.

For our purpose, we are faced with uniform matroid, which we argue that finding tight constraints is easy, i.e., we can compute the
\textsc{HitConstraint} subroutine in a very fast way. This subroutine is given a
fractional point $\yb$ and two variables $i$ and $j$, and tries to increase $y_i$
and decrease $y_j$ simultaneously, and find a new tight constraint $A$. For
sure, one should search for this new tight set through the sets having $i$
inside them but not $j$. So let $\Acal$ denote the family of all subsets
containing $i$ and not containing $j$. So we are interested in $\delta =
\min_{A\in \Acal} (r(A) - \yb(A))$, the maximum increase in $y_i$ (and decrease in
$y_j$) that does not violate any polytope condition, but produces a new tight
constraint. We claim that $\delta$ is trivial in case of the uniform matroid: $\delta = \min\{1-y_i, y_j\}$. Also the new tight set $A$ is either $\{i\}$ or $V - j$.

This simple form of  the \textsc{HitConstraint} gives an efficient
algorithm for \textsc{Randomized-Pipage-Rounding}, which we describe in
Algorithm~\ref{alg:pip}. Moreover, one has the following Theorem:
\begin{theorem}
	Let $\Mcal$ be the uniform matroid and $\yb$ be a fractional point inside $\Pcal(\Mcal)$. Then \textsc{Randomized-Pipage-Rounding} returns an integral point $\yb_{\text{rnd}}$ in $\Ocal(n)$ time, such that
	$$\EE[F(\yb_{\text{rnd}})] \geq F(\yb).$$
\end{theorem}
 The proof of this algorithm's correctness is similar to the original one given in  \cite{vondrak-pipage}. 
It is also noteworthy that our algorithm runs in
$\Ocal(n)$ time compared to $\Ocal(n^2)$, as described in \cite{vondrak-pipage}.

\begin{algorithm}[t]
   \caption{\textsc{Randomized-Pipage-Rounding} for the Uniform Matroid}
   \label{alg:pip}
\begin{algorithmic}[1]
\STATE \textbf{Input:} fractional $\yb$; $k\in \NN$ defining the matroid rank
\WHILE{$\yb$ fractional}
\STATE Select $i$ and $j$ among fractional variables
\IF{ $y_i + y_j < 1$}
\STATE Let $p = y_j/(y_i + y_j)$
\STATE With probability $p$, set $y_i \leftarrow 0$ and $y_j \leftarrow y_i + y_j$, and with probability $1-p$, set $y_i \leftarrow y_i + y_j$ and $y_j \leftarrow 0$.
\ELSE
\STATE Let $p = (1-y_i)/(2-y_i-y_j)$
\STATE With probability $p$, set $y_i \leftarrow y_i + y_j - 1$ and $y_j \leftarrow 1$, and with probability $1-p$, set $y_i \leftarrow 1$ and $y_j \leftarrow y_i+y_j-1$.
\ENDIF
\ENDWHILE
\STATE {\bfseries return} $\yb$
\end{algorithmic}
\end{algorithm}

\section{Details on experiments} \label{apx_details_experiments}
\subsection{Influence Maximization}\label{sec:det-inf}
Our  approach is to obtain samples from the product distribution, $(G, v)\sim \Gcal\times \Ucal$, and compute the set $P_v$ (see the definitions for the class of weighted coverage functions in Section~\ref{section3}). Note that the vertex $v$ is chosen uniformly at random. Since $P_v$ is smaller compared to $G$, it is less efficient to sample $G$ completely. Instead, while doing the BFS starting from $v$, we select edges with probability $p$ and proceed. Note that in case of maximizing the upper-bound \eqref{eq:ubound_inf}, whenever the sum of $x_i$ visited so far exceeds $1$, one can stop  and return $\zero$, otherwise return $\one_{P_v}$ in the end. This approach is quite fast, but may need  too many iterations to converge, because of its locality (i.e. we only take one vertex in each iteration). Note that the size of the gradient in this case is at most $\sqrt{|P_v|} \leq \sqrt{n}$.

\subsection{Facility Location}\label{sec:det-exe}
Computing the (stochastic) gradient for the concave upper bound can be done in linear time.
Let $h$ be the first index that $\sum_{j=1}^h x_j \geq 1$, then a vector in sub-gradient of $\Fbar(\cdot)$ is simply
\begin{align}\label{eq:ubound_grad_exe}
\gb = (m_1 - m_h, \ldots, m_{h-1} - m_h, 0, 0,\ldots, 0).
\end{align}

In case of the multilinear extension, we give a linear time algorithm for computing the gradient. Let $h$ be the first index that $x_h = 1$ (if no such index exists, then set $h = n+1$). It's clear from \eqref{eq:fac-mult} that $\frac{\partial F}{\partial x_i}(\xb) = 0$ for $i=h+1,\ldots,n$. For $i=h, h-1, \ldots, 1$ one has the following recursion:
$$\frac{\partial F_e}{\partial x_i}(\xb) = \frac{1-x_{i+1}}{1-x_i}\frac{\partial F_e}{\partial x_{i+1}}(\xb) + (m_{i} - m_{i+1})\prod_{j=1}^{i-1}(1-x_j),$$
which can be done completely in linear time.

\subsection{Exemplar-based Clustering}\label{appx:exem}
Let $V$ be a set of points. One can quantify the representativeness a set of exemplars $S\subseteq V$ by the loss function $L(S) = \frac{1}{|S|}\sum_{v\in V} \min_{s\in S} \|v - s\|_2$.  Finding the best $k$ exemplars is equivalent to solving $\min_{|S| = k} L(S)$. By introducing an appropriate phantom element $e_0$ we can turn $L(\cdot)$ into a monotone submodular function \cite{gomez-exem}: $f(S)= L(\{e_0\}) - L(S \cup \{e_0\})$. Thus maximizing $f$ is equivalent to minimizing $L$. In our experiments, to ensure non-negativity of the function values, we transform our dataset $V$ by the transformation $T:\RR^m \rightarrow \RR^m$,
\[
T(\xb) = \frac{3}{\sqrt{m}}\one + \frac{\xb - \bar{\xb}}{\|\xb - \bar{\xb}\|_2}, \quad \text{where}\; \bar{\xb} = \frac{1}{|V|}\sum_{\xb\in V} \xb,
\]
and set $e_0 = \zero$.
\section{Concave Envelope Evaluation}\label{appx:fenchel}
Here we prove Lemma~\ref{lem:upper}. Define $f(\xb)$ as follows:
\begin{align*}
	f(\xb) = 
	\begin{cases}
		1 - \prod_{i=1}^n (1-x_i) & \xb \in [0,1]^n \\
		-\infty & \text{Otherwise}
	\end{cases}.
\end{align*}
We first compute the Fenchel concave dual of $f$, which is defined as 
\begin{align}\label{eq:fenchel-dual}
	f_*(\yb) = \inf \{ \yb^\top \xb - f(\xb) \;|\; \xb \in \RR^n \}.
\end{align}
For brevity, let us define $h(\xb, \yb) := \yb^\top \xb - f(\xb)$. We partition $\RR^n$ into several subsets (cases below) and compute the infimum \eqref{eq:fenchel-dual} for each subset and take the minimum over all partitions.

\underline{Case I, $\xb \in (0,1)^n$:} Here we can compute the infimum by setting the gradient equal to zero. For a fixed $\yb\in\RR^n$ we have
\begin{align*}
	\nabla_\xb h(\xb, \yb) = \zero \iff y_i = \frac{\partial f}{\partial x_i} = \frac{\prod_{j=1}^n (1-x_j)}{(1-x_i)}.
\end{align*}
Clearly $y_i > 0$. Let us define $P = \prod_i (1-x_i)$. We then have $y_1\cdots y_n = P^{n-1}$, and then $x_i = 1 - P/y_i$. Since $x_i > 0$, we should have $y_i > P$. The following lemma gives  the necessary condition on $\yb$ for this to happen.
\begin{lemma}
	Let $y_1,\ldots,y_n \in \RR^+$, $n\geq 2$, and assume that $
		\forall i\in[n]:\; y_i > \sqrt[n-1]{y_1\cdots y_n}$. Then we have $y_i < 1$ for all $i\in[n]$.
\end{lemma}
\begin{proof}
	We prove the argument by induction. The case $n=2$ is obvious since
	\[ y_1 > y_1y_2 \Rightarrow y_2 < 1,\quad y_2 > y_1y_2 \Rightarrow y_1 < 1. \]
	Suppose the claim is true for $n-1$. We now prove it for $n$. W.l.o.g. assume $y_1\leq \cdots \leq y_n$. We have
	\[ y_1^{n-1} > y_1y_2\cdots y_n \Rightarrow \forall i\geq 2:\; y_i \geq y_1 > \sqrt[n-2]{y_2\cdots y_n}. \]
	So $y_2, \ldots, y_n$ satisfy the lemma's conditions, and by the induction hypothesis, we have $y_2,\ldots, y_n < 1$. Since $y_1 \leq y_2$, we also have $y_1 < 1$.
\end{proof}
So far, we know that there is a minimum in this case if $\yb\in(0,1)^n$. The minimum value would be
\begin{align}
	h(\xb, \yb) &= \sum x_iy_i - f(\xb) = \sum (1- P/y_i)y_i - (1 - P) \nonumber \\
		&= \sum y_i - (n-1)P - 1 \label{eq:min-over-interior}
\end{align}
This minimum value, as will be clear shortly, is not the best (lower values are available on other partitions), because of the following lemma:

\begin{lemma}
	Let $\yb\in(0,1)^n$. Then the minimum value of $h(\xb, \yb)$ over $\xb\in (0,1)^n$ is strictly greater than $-1$.
\end{lemma}
\begin{proof}
	Because of \eqref{eq:min-over-interior}, we have
	\begin{align*}
		\min_{\xb\in(0,1)^n} h(\xb, \yb) &= \sum_{i=1}^n y_i - (n-1)\sqrt[n-1]{y_1\cdots y_n} - 1 \\
		&> \sum_{i=2}^n y_i - (n-1)\sqrt[n-1]{y_2\cdots y_n} - 1 && \text{since $1 > y_1 > 0$} \\
		&\geq -1 && \text{by AM-GM inequality}
	\end{align*}
\end{proof}

\underline{Case II, at least for one index $i \in [n]$ we have $x_i = 1$:} In this case, we have $f(\xb) = 1$, and $h(\xb, \yb) = \yb^\top \xb - 1$. W.l.o.g. assume $y_1\leq \cdots \leq y_n$. It's clear that the minimum value for $h$ over these values of $\xb$ would be
\[
	\min_{\xb\in[0,1]^n, \exists_i x_i=1} h(\xb, \yb) = \begin{cases}
		y_1 - 1 & \yb \geq \zero \\
		\sum_{y_i < 0} y_i  - 1 & \text{Otherwise}
	\end{cases}.
\]

\underline{Case III, for some $i \in [n]$ we have $x_i = 0$:} In this case, $f$ would have the same form (with $x_i$ deleted) and also $y_i$ is deleted from $h$, and hence the problem is reduced to $n-1$ dimensional case. So the same type of solutions as the previous cases would occur.

In total, for $\yb$, such that $y_1\leq \cdots y_n$ we can write
\begin{align*}
	f_*(\yb) = \begin{cases}
		\sum_{y_i < 0} y_i - 1  & y_1 < 0 \\
		y_1 - 1 & 0 \leq y_1 < 1 \\
		0  &  1 \leq y_1
	\end{cases}
\end{align*}

Now that we have computed the Fenchel dual of $f$, we can compute the Fenchel dual of $f_*$, which would be the Fenchel bi-conjugate of $f$. By definition we have
\[ f_{**}(\zb) = \inf \{ \zb^\top \yb - f_*(\yb)\;|\; \yb\in \RR^n\}. \]
Let's define $g(\yb, \zb) = \zb^\top \yb - f_*(\yb)$. We will find the minimum on the orthant $y_1\leq \cdots \leq y_n$. 

\underline{Case IV, $y_1 \geq 1$:} We have $g(\yb, \zb) = \zb^\top \yb$. If $\zb$ has negative components, then infimum would become $-\infty$. So from now on, we assume $\zb \geq 0$. In this case, the minimum of $g$ over this case is $\sum z_i$.

\underline{Case V, $y_i \in (0,1)$:} We have 
	\[ g(\yb, \zb) = \zb^\top \yb -y_1 + 1 \geq y_1(\textstyle\sum z_i - 1) + 1. \]
	Now if $\sum z_i \geq 1$, the righthand side's minimum would be 1 and this minimum is achieved by setting $\yb = 0$. But if $\sum z_i < 1$, then
	\[ y_1(\textstyle \sum z_i - 1) + 1 \geq \sum z_i - 1 + 1 = \sum z_i, \]
	and this minimum is achieved by setting $\yb = \one$.
	
\underline{Case VI, $y_1 < 0$:} We have 
\begin{align*}
	g(\yb, \zb) &= \zb^\top \yb - \sum_{y_i < 0}y_i + 1 \\
		&= \sum_{i:y_i < 0} y_i(z_i - 1) + \sum_{j:y_j \geq 0} y_jz_j + 1
\end{align*}
If for some $i$, $z_i > 1$, then infimum of $g$ over this case would become $-\infty$, so we have also $\zb \leq \one$. In this case, the minimum would become $1$.

Summing all up, we have:
\begin{align*}
	f_{**}(\zb) = \begin{cases}
		\min \{1, \sum z_i\}  & \zb\in[0,1]^n \\
		-\infty  & \text{Otherwise}
	\end{cases},
\end{align*}
and this is what we wanted to prove.

\section{Pathological Examples}\label{apx:path}
Here we present a special case, where \textsc{Greedy} fails to give a proper solution, but our method works well. Our example would be about influence maximization with partition matroid condition. It is well known that for general matroids (matroids other than the uniform matroid), \textsc{Greedy} is guaranteed to give $1/2$-optimal solution. 

Construct a graph $G = (V, E)$ as follows: $V = \{0,1,2,\ldots, 2N+1\}$, and connect vertices like in the figure. Also take the partitions to be $\{0\}$ and $\{1,2,\ldots,2N+1\}$, meaning that one should take from each partition at most one vertex. 
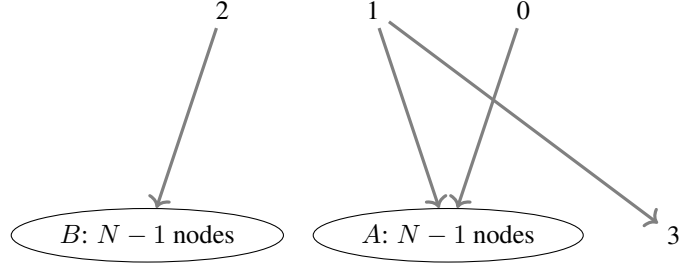
\begin{figure}[hpt]
	\begin{center}
		\begin{tikzpicture}
			\usetikzlibrary{shapes}
			\usetikzlibrary{arrows.meta}
			\begin{scope}[every node/.style={fill=none,draw,ellipse, minimum width=100pt}]
			    \node (A) at (2,0) {$B$: $N-1$ nodes};
			    \node (B) at (6,0) {$A$: $N-1$ nodes};
			\end{scope}
			\node (C) at (9,0) {3};
			\begin{scope}[every node/.style={}]
			    \node (v3) at (7,3) {0};
			    \node (v2) at (5,3) {1};
			    \node (v1) at (3,3) {2};
			\end{scope}
			\begin{scope}[every node/.style={fill=white,circle},
			              every edge/.style={draw=gray,very thick}]
			    \path [->] (v1) edge (A);
			    \path [->] (v2) edge (B);
			    \path [->] (v3) edge (B);
			    \path [->] (v2) edge (C);
			\end{scope}
		\end{tikzpicture}
	\end{center}
	\caption{Pathological graph for Influence Maximization}
\end{figure}

The \textsc{Greedy} algorithm, first chooses the vertex with the highest out-degree, which is $1$, and then is forced to choose $0$ as the second vertex because of matroid condition, leading to the $(1/2+\epsilon)$-optimal answer $\{0, 1\}$.

On the other hand, our algorithm successfully gives the optimal answer $\{0, 2\}$. It's a good practice to show why. Let us choose $\xb^{(0)}$ be the projection of $\one$ on the partition matroid polytope, namely
\[ \xb^{(0)} = (1, \tfrac{1}{2N+1}, \ldots, \tfrac{1}{2N+1}). \]
The (sub-)gradient of $\Fbar$ at $\xb^{(0)}$ is calculated as follows:
\begin{align*}
	\nabla \Fbar(\xb^{(0)}) &= \zero_{\{0\}} + \one_{\{1\}} +\one_{\{2\}} + \one_{\{3,1\}} + \sum_{a\in A} \zero_{\{0,1,a\}} + \sum_{b\in B} \one_{\{2, b\}} \\
	&= (0, 2, N, 1, \underbrace{0, \ldots, 0}_{a\in A},  \underbrace{1, \ldots, 1}_{b\in B})
\end{align*}
where the reason of $\zero_{\{0\}}$ and $\zero_{\{0,1,a\}}$ for $a\in A$ is $\xb_0 = 1$ and $\xb_0 + \xb_1 + \xb_a > 1$ respectively. It's obvious that moving along this gradient, and projecting, makes $\xb_0 = \xb_2 = 1$ and all others to zero, selecting $\{0, 2\}$ as the solution.

The difference here with \textsc{Greedy} is apparently in two places: (i) being on the matroid base polytope forces the algorithm tochoose a vertex from  the first partition, and simultaneously (ii) selecting $0$ means all vertices $a\in A$ are influenced, so there is no need to select $1$. \textsc{Greedy} does not take into account that in the future it should select a node in some other partition, that may lose his achievement in the first step.

\end{document}